\documentclass{article}
\usepackage{microtype}
\usepackage{graphicx}
\usepackage{subfigure}
\usepackage{booktabs} 

\usepackage{hyperref}

\usepackage{times}
\usepackage{natbib}

\usepackage{algorithm}
\usepackage{algorithmic}
\usepackage{amsmath}
\usepackage{amssymb}
\usepackage{amsthm}
\DeclareFontFamily{U}{mathx}{\hyphenchar\font45}
\DeclareFontShape{U}{mathx}{m}{n}{
      <5> <6> <7> <8> <9> <10>
      <10.95> <12> <14.4> <17.28> <20.74> <24.88>
      mathx10
      }{}
\DeclareSymbolFont{mathx}{U}{mathx}{m}{n}
\DeclareFontSubstitution{U}{mathx}{m}{n}
\DeclareMathAccent{\widecheck}{0}{mathx}{"71}

\newtheorem{theorem}{Theorem}
\newtheorem{lemma}{Lemma}

\newcommand{\argmax}{\mathop{\rm argmax}}
\newcommand{\argmin}{\mathop{\rm argmin}}
\newcommand{\rmax}{R_{\max}}
\newcommand{\eps}{\epsilon}
\newcommand{\mc}{\mathcal}
\newcommand{\norm}[1]{\|#1\|_{\infty}}
\renewcommand{\check}{\widecheck}
\renewcommand{\hat}{\widehat}

\usepackage[accepted]{icml2018}

\icmltitlerunning{Mitigating Planner Overfitting in Model-Based Reinforcement Learning}

\icmltitlerunning{Mitigating Planner Overfitting in Model-Based Reinforcement Learning}

\begin{document} 
\twocolumn[
\icmltitle{Mitigating Planner Overfitting in Model-Based Reinforcement Learning}


\begin{icmlauthorlist}
\icmlauthor{Dilip Arumugam}{st}
\icmlauthor{David Abel}{br}
\icmlauthor{Kavosh Asadi}{br}
\icmlauthor{Nakul Gopalan}{br}
\icmlauthor{Christopher Grimm}{um}
\icmlauthor{Jun Ki Lee}{br}
\icmlauthor{Lucas Lehnert}{br}
\icmlauthor{Michael L. Littman}{br}
\end{icmlauthorlist}

\icmlaffiliation{st}{Department of Computer Science, Stanford University}
\icmlaffiliation{br}{Department of Computer Science, Brown University}
\icmlaffiliation{um}{Department of Computer Science \& Engineering, University of Michigan}

\icmlcorrespondingauthor{Dilip Arumugam}{dilip@cs.stanford.edu}


\vskip 0.3in
]


\printAffiliationsAndNotice{}

\begin{abstract} 
An agent with an inaccurate model of its environment faces a difficult choice: it can ignore the errors in its model and act in the real world in whatever way it determines is optimal with respect to its model.  Alternatively, it can take a more conservative stance and eschew its model in favor of optimizing its behavior solely via real-world interaction. This latter approach can be exceedingly slow to learn from experience, while the former can lead to ``planner overfitting''---aspects of the agent's behavior are optimized to exploit errors in its model. This paper explores an intermediate position in which the planner seeks to avoid overfitting through a kind of regularization of the plans it considers. We present three different approaches that demonstrably mitigate planner overfitting in reinforcement-learning environments.
\end{abstract} 

\section{Introduction}
\label{s:intro}

Model-based reinforcement learning (RL) has proven to be a powerful approach for generating reward-seeking behavior in sequential decision-making environments. For example, a number of methods are known for guaranteeing near optimal behavior in a Markov decision process (MDP) by adopting a model-based approach~\cite{kearns98,brafman02,strehl09}. In this line of work, a learning agent continually updates its model of the transition dynamics of the environment and actively seeks out parts of its environment that could contribute to achieving high reward but that are not yet well learned. Policies, in this setting, are designed specifically to explore unknown transitions so that the agent will be able to exploit (that is, maximize reward) in the long run.

A distinct model-based RL problem is one in which an agent has explored its environment, constructed a model, and must then use this learned model to select the best policy that it can. A straightforward approach to this problem, referred to as the \emph{certainty equivalence approximation}~\cite{dayan96}, is to take the learned model and to compute its optimal policy, deploying the resulting policy in the real environment. The promise of such an approach is that, for environments that are defined by relatively simple dynamics but require complex behavior, a model-based learner can start making high-quality decisions with little data.

Nevertheless, recent large-scale successes of reinforcement learning have not been due to model-based methods but instead derive from value-function based or policy-search methods~\cite{Mnih2015HumanlevelCT,Mnih2016AsynchronousMF,Schulman2017ProximalPO,Hessel2018RainbowCI}.
Attempts to leverage model-based methods have fallen below expectations, particularly when models are learned using function-approximation methods. \citet{jiang15} highlighted a significant shortcoming of the certainty equivalence approximation, showing that it is important to hedge against possibly misleading errors in a learned model. They found that reducing the effective planning depth by decreasing the discount factor used for decision making can result in improved performance when operating in the true environment.

At first, this result might seem counter intuitive---the best way to exploit a learned model can be to exploit it incompletely. However, an analogous situation arises in supervised machine learning. It is well established that, particularly when data is sparse, the representational capacity of supervised learning methods must be restrained or regularized to avoid overfitting. Returning the best hypothesis in a hypothesis class relative to the training data can be problematic if the hypothesis class is overly expressive relative to the size of the training data. The classic result is that testing performance improves, plateaus, then drops as the complexity of the learner's hypothesis class is increased.

In this paper, we extend the results on avoiding planner overfitting via decreasing discount rates by introducing several other ways of regularizing policies in model-based RL. In each case, we see the classic ``overfitting'' pattern in which resisting the urge to treat the learned model as correct and to search in a reduced policy class is repaid by improved performance in the actual environment. We believe this research direction may hold the key to large-scale applications of model-based RL.

Section~\ref{s:definitions} provides a set of definitions, which provide a vocabulary for the paper. Section~\ref{s:gamma} reviews the results on decreasing discount rates, Section~\ref{s:epsilon} presents a new approach that plans using epsilon greedy policies, and Section~\ref{s:search} presents results where policy-search is performed using lower capacity representations of policies. Section~\ref{s:related} summarizes related work and Section~\ref{s:conclusions} concludes.

\section{Definitions}
\label{s:definitions}

An MDP $M$ is defined by the quantities $\langle S, A, R, T, \gamma \rangle$, where $S$ is a state space, $A$ is an action space, $R:S\times A \rightarrow \mathbb{R}$ is a reward function, $T:S\times A \rightarrow \mathbb{P}(S)$ is a transition function, and $0\le \gamma < 1$ is a discount factor. The notation $\mathbb{P}(X)$ represents the set of probability distributions over the discrete set $X$. Given an MDP $M=\langle S, A, R, T, \gamma \rangle$, its optimal value function $Q^*$ is the solution to the Bellman equation: $$Q^*(s,a) = R(s,a) + \gamma \sum_{s'} T(s,a)_{s'} \max_{a'} Q^*(s',a').$$ This function is unique and can be computed by algorithms such as value iteration or linear programming~\cite{Puterman94}.

A (deterministic) policy is a mapping from states to actions, $\pi:S\rightarrow A$. Given a value function $Q: S\times A \rightarrow \mathbb{R}$, the \emph{greedy policy} with respect to $Q$ is $\pi_Q(s) = \argmax_a Q(s,a)$. The greedy policy with respect to $Q^*$ maximizes expected discounted reward from all states. We assume that ties between actions of the greedy policy are broken arbitrarily but consistently so there is always a unique optimal policy for any MDP.

The \emph{value function for a policy} $\pi$ \emph{deployed in} $M$ can be found by solving $$Q^\pi_M(s,a) = R(s,a) + \gamma \sum_{s'} T(s,a)_{s'} Q^\pi_M(s',\pi(s')).$$ The value function of the optimal policy is the optimal value function. For a policy $\pi$, we also define the scalar $V^\pi_M = \sum_s w_s Q^\pi_M(s,\pi(s))$, where $w$ is an MDP-specific weighting function over the states. 

The \emph{epsilon-greedy policy}~\cite{sutton98} is a stochastic policy where the probability of choosing action $a$ is $(1-\eps) + \eps/|A|$ if $a = \argmax_a Q(s,a)$ and $\eps/|A|$ otherwise.  The optimal epsilon greedy policy for $M$ is not generally the epsilon greedy policy for $Q^*$. Instead, it is necessary to solve a different set of Bellman equations:
\begin{eqnarray*}
\lefteqn{Q^\eps(s,a) = R(s,a) + \gamma \sum_{s'} T(s,a)_{s'}\; \times} \\
& &   \left( (1-\eps) \max_{a'} Q^\eps(s',a') + \eps/|A| \sum_{a'} Q^\eps(s',a')\right).
\end{eqnarray*}
The optimal epsilon-greedy policy plays an important role in the analysis of learning algorithms like SARSA~\cite{rummery94,littman96}.

These examples of optimal policies are with respect to all possible deterministic Markov policies. In this paper, we also consider optimization with respect to a restricted set of policies $\widecheck{\Pi}$. The optimal restricted policy can be found by comparing the scalar values of the policies: $\rho^* = \argmax_{\rho\in\check{\Pi}} V_\rho.$

\section{Decreased Discounting}
\label{s:gamma}

Let $M=\langle S, A, R, T, \gamma \rangle$ be the evaluation environment and $\hat{M}=\langle S, A, R, \hat{T}, \check{\gamma} \rangle$ be the planning environment, where $\hat{T}$ is the learned model and $\check{\gamma}\le \gamma$ is a smaller discount factor used to decrease the effective planning horizon.

\citet{jiang15} proved a bound on the difference between the performance of the optimal policy in $M$ and the performance of the optimal policy in $\hat{M}$ when executed in $M$:
\begin{equation}
\frac{\gamma-\check{\gamma}}{(1-\gamma)(1-\check{\gamma})} \rmax +
   \frac{2 \rmax}{(1-\check{\gamma})^2} \sqrt{\frac{1}{2n} \log{\frac{2 |S| |A| |\Pi_{R,\check{\gamma}}|}{\delta}}}.
\label{e:gammabound}
\end{equation}
Here, $\rmax = \max_{s,a} R(s,a)$ is the largest reward (we assume all rewards are non-negative), $\delta$ is the certainty with which the bound needs to hold, $n$ is the number of samples of each transition used to build the model, and $|\Pi_{R,\check{\gamma}}|$ is the number of distinct possibly optimal policies for $\langle S, A, R, \cdot, \check{\gamma} \rangle$ over the entire space of possible transition functions.

They show that $|\Pi_{R,\check{\gamma}}|$ is an increasing function of $\check{\gamma}$, growing from 1 to as high as $|A|^{|S|}$, the size of the set of all possible deterministic policies. They left open the shape of this function, which is most useful if it grows gradually, but could possibly jump abruptly.


To help ground intuitions, we estimated the shape of $|\Pi_{R,\check{\gamma}}|$ over a set of randomly generated MDPs. Following \citet{jiang15}, a ``ten-state chain'' MDP $M = \langle S, A, T, R, \gamma\rangle$ is drawn such that, for each state--action pair, $(s,a) \in S \times A$, the transition function $T(s,a)$ is constructed by choosing $5$ states at random from $S$, then assigning probabilities to these states by drawing $5$ independent samples from a uniform distribution over $[0,1]$ and normalizing the resulting numbers.  The probability of transition to any other state is zero. For each state--action pair $(s,a)\in S \times A$, the reward $R(s,a)$ is drawn from a uniform distribution with support $[0,1]$. For our MDPs, we chose $|S| = 10$, $|A| = 2$ and $\gamma = 0.99$. We examined $\check{\gamma}$ in $\{0.0, 0.1,...,0.9, 0.99\}$,
computed optimal policies by running value iteration with $10$ iterations. We sampled repeatedly until no new optimal policy was discovered for $5000$ consecutive samples.


Figure~\ref{f:Pigamma} is an estimate of how $|\Pi_{R,\check{\gamma}}|$ grows in this class of randomly generated MDPs. Fortunately, the set appears to grow gradually, making $\check{\gamma}$ an effective parameter for fighting planner
overfitting.

Estimating $|\Pi_{R,\check{\gamma}}| \approx 11 e^{\check{\gamma}}-10$, Figure~\ref{f:gammabound} shows the bound of Equation~\ref{e:gammabound} applied to the random MDP distribution ($|S|=10$, $|A|=2$, $\rmax=1$, $\gamma=.99$).

Note that the expected ``U'' shape is visible, but only for a relatively narrow range of values of $n$. For under $50$k samples, the minimal loss bound is achieved for $\check{\gamma} = 0$.  For over $900$k samples, the minimal loss bound is achieved for $\check{\gamma}=\gamma$. (Note that the pattern shown here is relatively insensitive to the estimated shape of $|\Pi_{R,\check{\gamma}}|$.)


\begin{figure}
\centering
\includegraphics[width=3.5in]{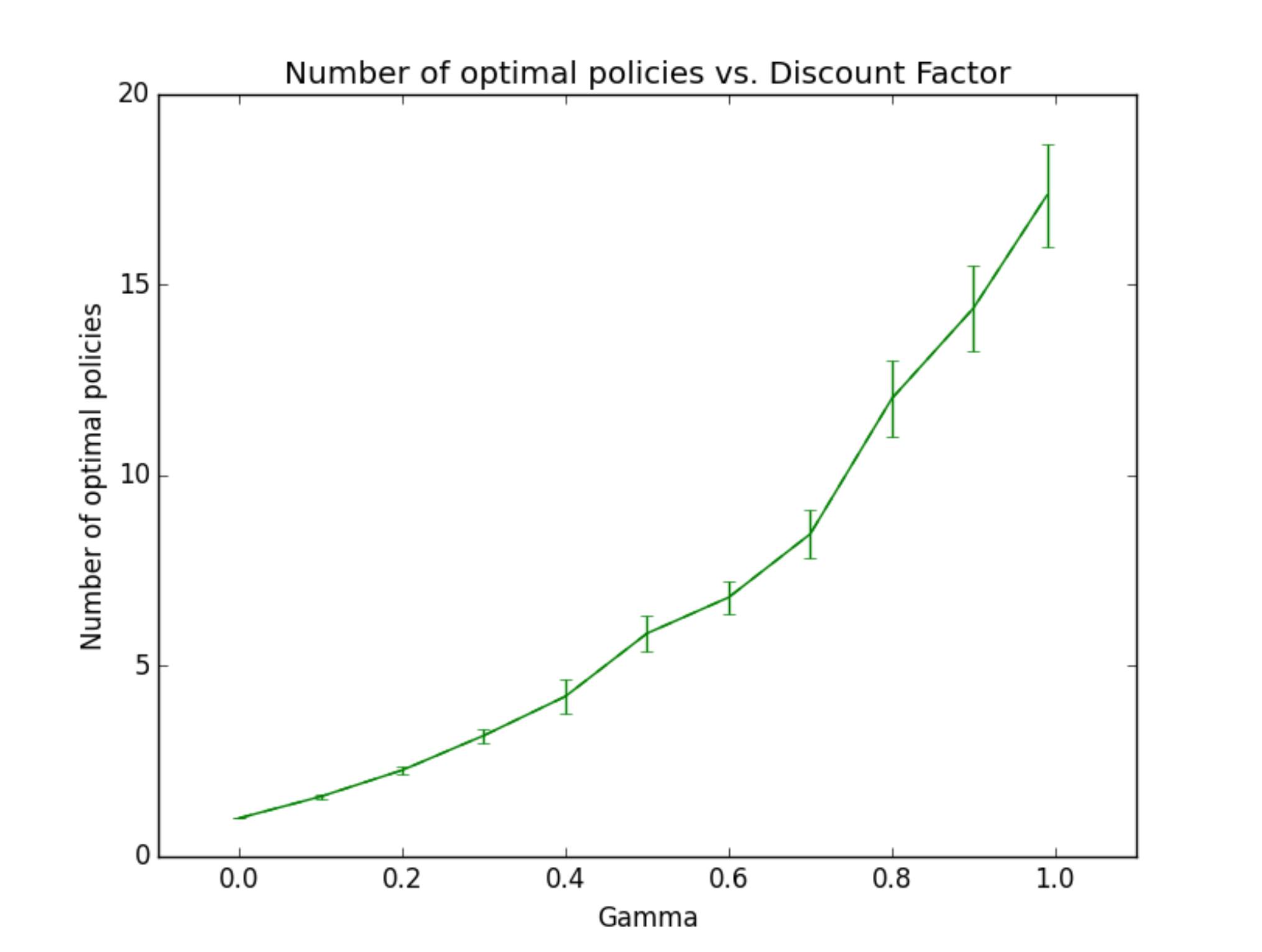}
\caption{The number of distinct optimal policies found generating random transition functions for a fixed reward function varying $\check{\gamma}$ in random MDPs.}
\label{f:Pigamma}
\end{figure}

\begin{figure}
\centering
\includegraphics[width=3.5in]{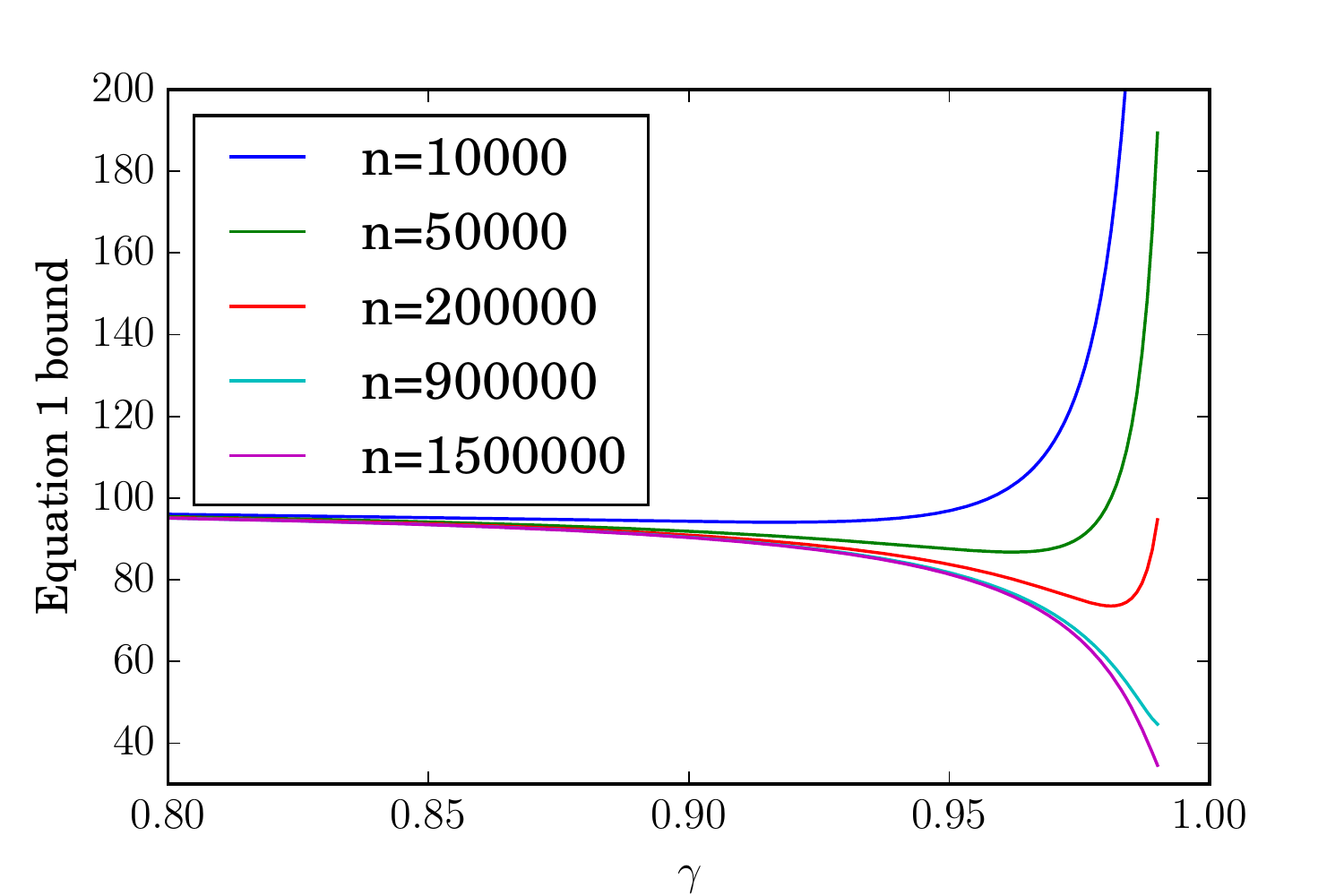}
\caption{Bound on policy loss for randomly generated MDPs, showing the tightest bound for intermediate values of $\gamma$ for intermediate amounts of data.}
\label{f:gammabound}
\end{figure}

For actual MDPs, the ``U'' shape is much more robust. Using this same distribution over MDPs, Figure~\ref{f:gamma} replicates an empirical result of \citet{jiang15} showing that intermediate values of $\check{\gamma}$ are most successful and that this value grows as the model used in planning becomes more accurate (having been trained on more trajectories). We sampled MDPs from the same random distribution and, for each value of $n \in \{5, 10, 20, 50\}$, we generated $1000$ datasets each consisting of $n$ trajectories of length $10$ starting from a state selected uniformly at random and executing a random policy. In all experiments, the estimated MDP ($\hat{M}$) was computed using maximum likelihood estimates of $T$ and $R$ with no additive Gaussian noise. Optimal policies were all found by running value iteration in the estimated MDP $\hat{M}$. The empirical loss (Equation~14 of \citet{jiang15}) was computed for each value of $\gamma \in \{0.0, 0.1, 0.2, \ldots, 0.9, 0.99\}$.  The error bars shown in the figure represent $95$\% confidence intervals.

\begin{figure}
\centering
\includegraphics[width=3.5in]{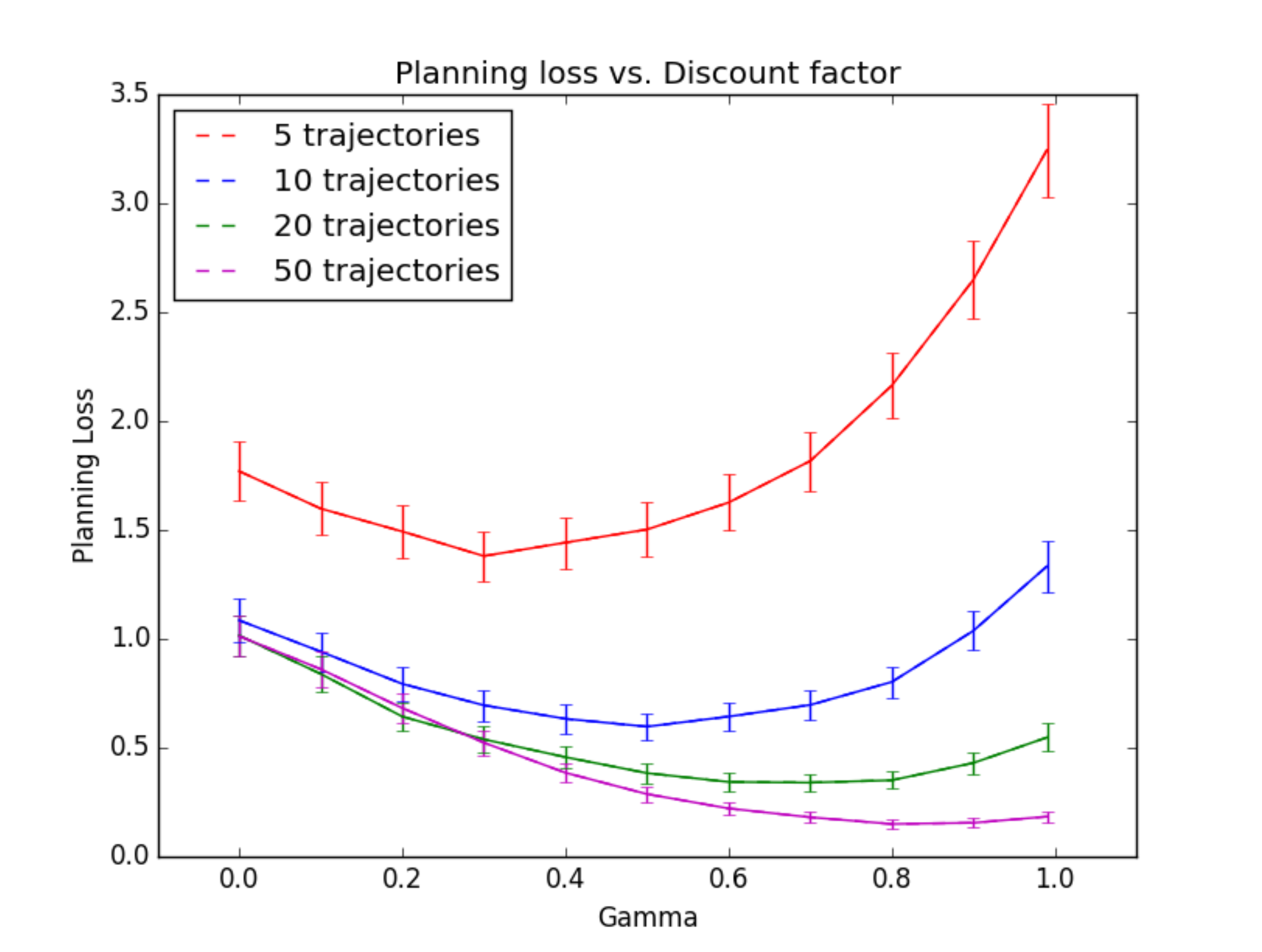}
\caption{Reducing the discount factor used in planning combats planner overfitting in random MDPs.}
\label{f:gamma}
\end{figure}


\section{Increased Exploration}
\label{s:epsilon}

In this section, we consider a novel regularization approach in which planning is performed over the set of epsilon-greedy policies. The intuition here is that adding noise to the policies makes it harder for them to be tailored explicitly to the learned model, resulting in less planner overfitting.

In Section~\ref{s:bound}, a general bound is introduced and then Section~\ref{s:epsbound} applies the bound to the set of epsilon greedy policies.

\subsection{General Bounds}
\label{s:bound}

We can relate the structure of a restricted set of policies $\check{\Pi}$ to the performance in an approximate model with the following theorem.
 

\begin{theorem}
Let $\check{\Pi}$ be a set of policies for an MDP $M=\langle S, A, T, R, \gamma \rangle$. Let $\hat{M}=\langle S, A, \hat{T}, R, \gamma \rangle$ be an MDP like $M$, but with a different transition function. Let $\pi$ be the optimal policy for $M$ and $\hat{\pi}$ be the optimal policy for $\hat{M}$. Let $\rho$ be the optimal policy in $\check{\Pi}$ for $M$ and $\hat{\rho}$ be the optimal policy in $\check{\Pi}$ for $\hat{M}$. Then,
$$|V^{\pi}_M - V^{\hat{\rho}}_M| \le |V^{\pi}_M - V^{\rho}_M| +  2 \max_{p\in\check{\Pi}} |V^p_M - {V}^p_{\hat{M}} |.$$
\label{th:general}
\end{theorem}

\begin{proof}
We can write
\begin{eqnarray}
\lefteqn{V^{\pi}_M - V^{\hat{\rho}}_M}
\nonumber \\
& =& (V^{\pi}_M - V^{\rho}_M) + (V^{\rho}_M - {V}^{\rho}_{\hat{M}}) \nonumber \\
& & - (V^{\hat{\rho}}_M-{V}^{\hat{\rho}}_{\hat{M}}) - ({V}^{\hat{\rho}}_{\hat{M}} - {V}^{\rho}_{\hat{M}} )
\nonumber \\
&\le& (V^{\pi}_M - V^{\rho}_M) + (V^{\rho}_M - {V}^{\rho}_{\hat{M}}) - (V^{\hat{\rho}}_M-{V}^{\hat{\rho}}_{\hat{M}})
\label{e:opt} \\
&\le& |V^{\pi}_M - V^{\rho}_M| + |V^{\rho}_M - {V}^{\rho}_{\hat{M}}| + |V^{\hat{\rho}}_M-{V}^{\hat{\rho}}_{\hat{M}}|
\nonumber \\
&\le& |V^{\pi}_M - V^{\rho}_M| + 2 \max_{p\in\check{\Pi}} |V^p_M - {V}^p_{\hat{M}}|.
\label{e:restricted}
\end{eqnarray}
Equation~\ref{e:opt} follows from the fact that ${V}^{\hat{\rho}}_{\hat{M}} - {V}^{\rho}_{\hat{M}} \ge 0$, since $\hat{\rho}$ is chosen as optimal among the set of restricted policies with respect to $\hat{M}$. Equation~\ref{e:restricted} follows because both $\rho$ and $\hat{\rho}$ are included in $\check{\Pi}$.  The theorem follows from the fact that $V^{\pi}_M - V^{\hat{\rho}}_M \ge 0$ since $\pi$ is chosen to be optimal in $M$.
\end{proof}

Theorem~\ref{th:general} shows that the restricted policy set $\check{\Pi}$ impacts the resulting value of the plan in two ways. First, the bigger the class is, the closer $V^{\rho}_M$ becomes to $V^{\pi}_M$---that is, the more policies we consider, the closer to optimal we become. At the same time, $\max_{p\in\check{\Pi}} |V^p_M - {V}^p_{\hat{M}}|$ grows as $\check{\Pi}$ gets larger as there are more policies that can differ in value between $M$ and $
\hat{M}$.

\citet{jiang15} leverage this structure in the specific case of defining $\check{\Pi}$ by optimizing policies using a smaller value for $\gamma$. Our Theorem~\ref{th:general} generalizes the idea to arbitrary restricted policy classes and arbitrary pairs of MDPs $M$ and $\hat{M}$.

In particular, Consider a sequence of $\Pi_i$ such that $\Pi_i \subseteq\; \Pi_{i+1}$. Then, the first part of the bound is monotonically non-increasing (it goes down each time a better policy is included in the set) and the second part of the bound is monotonically non-decreasing (it goes up each time a policy is included that magnifies the difference in performance possible in the two MDPs).

In Lemma~\ref{l:simulation}, we show that the particular choice of $\hat{M}$ that comes from statistically sampling transitions as in certainty equivalence leads to a bound on $|V^p_M - {V}^p_{\hat{M}}|$, for an arbitrary policy $p$.

\begin{lemma}
Given true MDP $M$,
let $\hat{M}$ be an MDP comprised of a reward function $R$ and transition function $\hat{T}$ estimated from $n$ samples for each state--action pair, and let $p$ be a policy, then the following holds with probability at least $1-\delta$:
$$                                                   
2 \max_{p\in\check{\Pi}}|V^p_M - {V}^p_{\hat{M}}| \le                        
 \frac{2 \rmax}{(1-\gamma)^2} \sqrt{\frac{1}{2n} \log \frac{2|S||A||\check{\Pi}|}{\delta}}.
$$
\label{l:simulation}
\end{lemma}

\begin{proof}
This lemma is a variation of the classic ``Simulation Lemma''~\cite{kearns98,strehl09} and is proven in this form as Theorem 2 of \citet{jiang15} (specifically, Lemma 2).
Note that their proof, though stated with respect to a particular choice of $\check{\Pi}$ set, holds in this general form.
\end{proof}

\subsection{Bound for Epsilon-Greedy Policies}
\label{s:epsbound}

It remains to show that $\norm{V^{\pi}_M - V^{\hat{\rho}}_M}$ is bounded when restricted to epsilon-greedy policies. For the case of planning with a decreased discount factor, \citet{jiang15} provide a bound for this quantity in their Lemma~1. For the case of epsilon-greedy policies, the corresponding bound is proven in the following lemma.

\begin{lemma}
\label{l:soft}
For any MDP $M$, the difference in value of the optimal policy $\pi$ and the optimal $\eps$-greedy policy $\rho$ is bounded by:
\begin{equation}
| V_M^{\pi} - V_M^{\rho} | \le \rmax \frac{\eps}{(1-\gamma)(1 - \gamma (1-\eps))}. \nonumber
\end{equation}
\end{lemma}
\begin{proof}
Let $\pi$ be the optimal policy for $M$ and $u$ be a policy that selects actions uniformly at random. We can define $\pi_\eps$, an $\eps$-greedy version of $\pi$, as:
\begin{equation}
\pi_\eps^a (s) = (1-\eps) \pi^a(s) + \eps u^a(s).
\end{equation}
where $\pi^a(s)$ refers to the probability associated with action $a$ under a policy $\pi$.
Let $T^\pi$ denote the transition matrix from states to states under policy $\pi$. Using the above definition, we can decompose the transition matrix into
\begin{equation}
T^{\pi_\eps} = (1-\eps) T^{\pi} + \eps T^{u}.
\end{equation}
Similarly, we have for the reward vector over states,
\begin{equation}
R^{\pi_\eps} = (1-\eps) R^{\pi} + \eps R^{u}.
\end{equation}
To obtain our bound, note
\begin{eqnarray}
\lefteqn{V_M^{\pi} - V_M^{\pi_\eps}} \nonumber \\
&=& \sum_{t=1}^\infty \gamma^{t-1} \left[ T^{\pi} \right]^{t-1} R^{\pi} - \gamma^{t-1}  \left[ T^{\pi_\eps} \right]^{t-1} R^{\pi_\eps} \nonumber \\
&=& \sum_{t=1}^\infty \gamma^{t-1} \left[ T^{\pi} \right]^{t-1} R^{\pi} \nonumber \\&&- \gamma^{t-1}  \left[ T^{\pi_\eps} \right]^{t-1} \left[ (1-\eps) R^{\pi} + \eps R^{u} \right] \nonumber \\
&=& \sum_{t=1}^\infty \gamma^{t-1} \left[ T^{\pi} \right]^{t-1} R^{\pi} \nonumber \\&&- \gamma^{t-1}  \left[ T^{\pi_\eps} \right]^{t-1} (1-\eps) R^{\pi} - \underbrace{\gamma^{t-1}  \left[ T^{\pi_\eps} \right]^{t-1} \eps R^{u}}_{\ge 0} \nonumber \\
&\le&  \sum_{t=1}^\infty \gamma^{t-1} \left[ T^{\pi} \right]^{t-1} R^{\pi} \nonumber \\
&&- \gamma^{t-1}  \left[ (1-\eps) T^{\pi} + \eps T^{u} \right]^{t-1} (1-\eps) R^{\pi} . \label{eq:lem-2-pf-1}
\end{eqnarray}
Since $T^\pi$ is a transition matrix, all its entries lie in $[0,1]$; hence, we have the following element-wise matrix inequality:
\begin{equation}
 \left[ (1-\eps) T^{\pi} + \eps T^{u} \right]^{t-1} \ge  \left[ (1-\eps) T^{\pi} \right]^{t-1} . \label{eq:lem-2-pf-2}
\end{equation}
Plugging inequality~\ref{eq:lem-2-pf-2} into the bound~\ref{eq:lem-2-pf-1} results in
\begin{eqnarray*}
\lefteqn{V_M^{\pi} - V_M^{\pi_\eps}} \\
&\le& \sum_{t=1}^\infty \gamma^{t-1} \left[ T^{\pi} \right]^{t-1} R^{\pi} - \gamma^{t-1}  (1-\eps)^t
 \left[ T^{\pi} \right]^{t-1} R^{\pi} \\
&\le& \norm{ \sum_{t=1}^\infty \gamma^{t-1} (1 - (1-\eps)^t ) \left[ T^{\pi} \right]^{t-1} R^{\pi} } \\
\end{eqnarray*}
Since $\norm{ R^\pi } = \rmax$ we can upper bound the norm of difference of the values vector 
over states with
\begin{eqnarray*}
\norm{V_M^{\pi} - V_M^{\pi_\eps}}
&\le& \sum_{t=1}^\infty \gamma^{t-1} (1 - (1-\eps)^t ) \rmax \\
&=& \frac{\eps}{(1-\gamma) (\eps \gamma - \gamma + 1)} \rmax.
\end{eqnarray*}
Using this inequality, we can bound the difference in value of the optimal policy $\pi$ and the optimal $\eps$-greedy policy $\rho$ by
\begin{eqnarray*}
\norm{ V_M^{\pi} - V_M^{\rho} } &\le& \norm{ V_M^{\pi} - V_M^{\pi_\eps} } \\
&\le& \frac{\eps}{(1-\gamma) (\eps \gamma - \gamma + 1)} \rmax.
\end{eqnarray*}
\end{proof}

Figure~\ref{f:Piepsilon} is an estimate of how $|\Pi_{R,\epsilon}|$ grows over the class of randomly generated MDPs. Again, the set appears to grow gradually as $\epsilon$ decreases, making $\epsilon$ another effective parameter for fighting planner overfitting.

\begin{figure}
\centering
\includegraphics[width=3.5in]{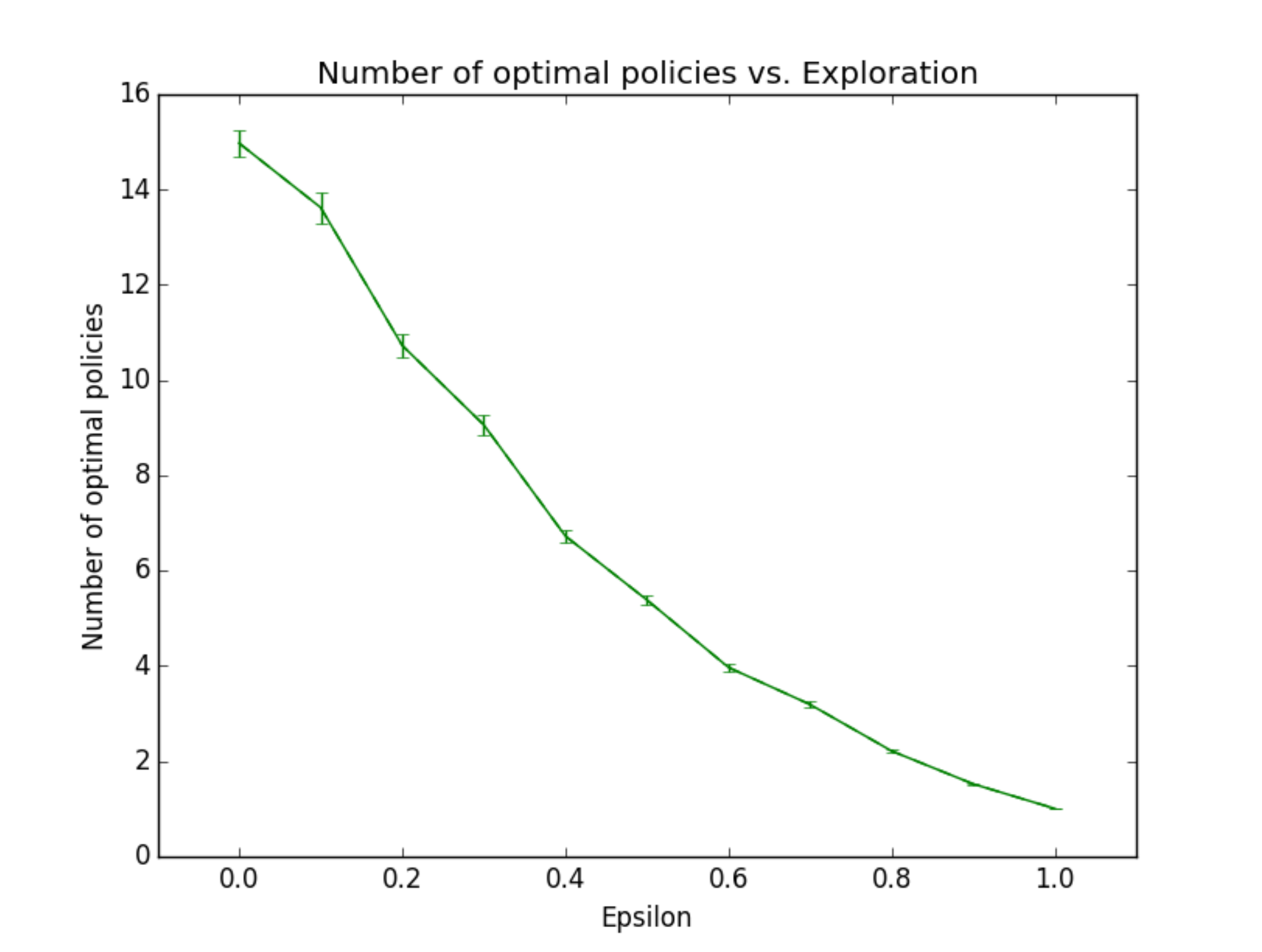}
\caption{The number of distinct optimal policies found generating random transition functions for a fixed reward function varying $\epsilon$.}
\label{f:Piepsilon}
\end{figure}




\subsection{Empirical Results}

We evaluated this exploration-based regularization approach in the distribution over MDPs used in Figure~\ref{f:gamma}. Figure~\ref{f:epsilon} shows results for each value of $\epsilon \in \{0.0, 0.1, 0.2, \ldots, 0.9, 1.0\}$. Here, the maximum likelihood transition function $\hat{T}$ was replaced with the epsilon-softened transition function $T_{\epsilon}$. In contrast to the previous figure, regularization increases as we go to the right.  Once again, we see that intermediate values of $\epsilon$ are most successful and the best value of $\epsilon$ decreases as the model used in planning becomes more accurate (having been trained on more trajectories). The similarity to Figure~\ref{f:gamma} is striking---in spite of the difference in approach, it is essentially the mirror image of Figure~\ref{f:gamma}.

\begin{figure}
\centering
\includegraphics[width=3.5in]{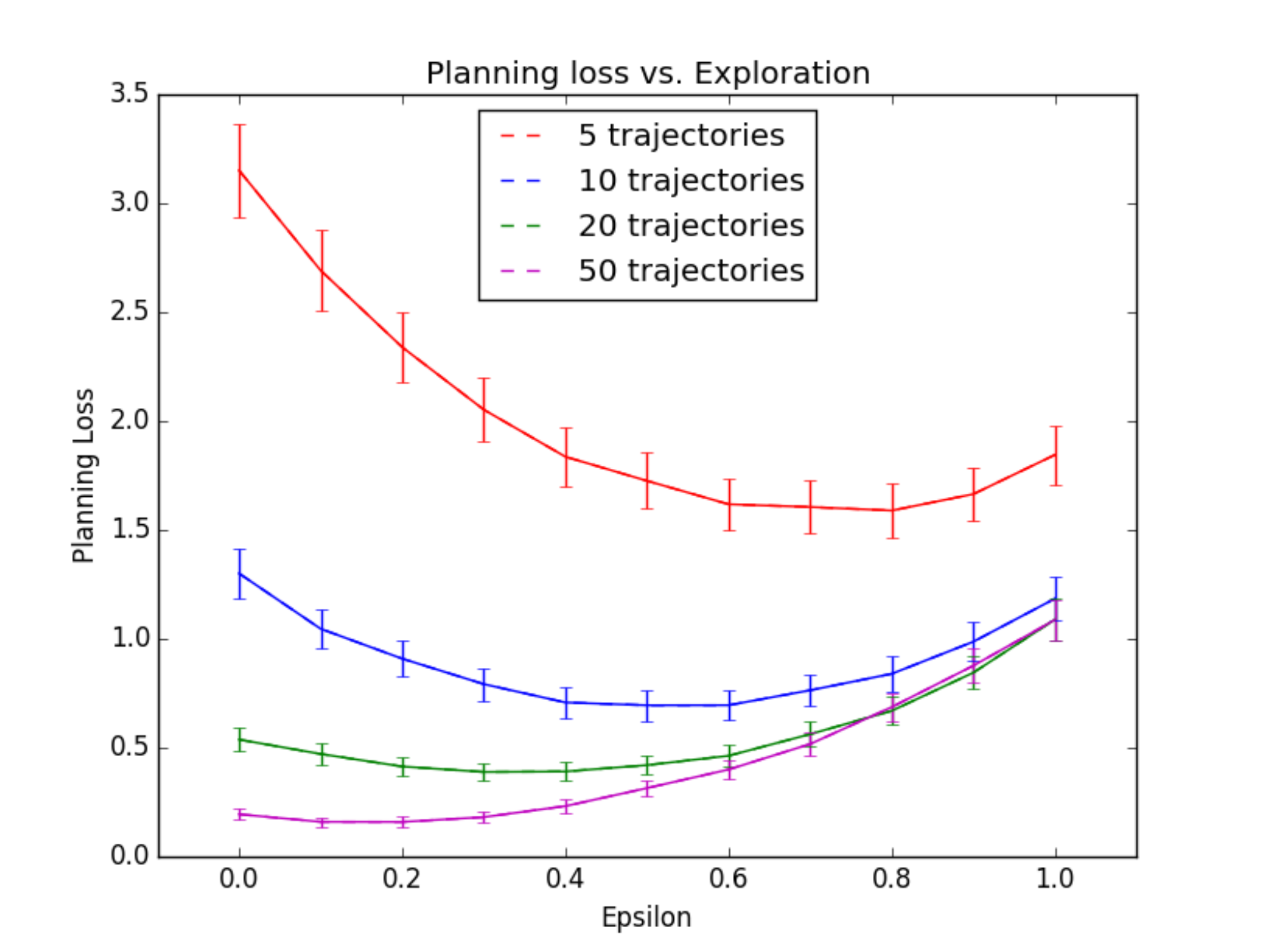}
\caption{Increasing the randomness in action selection during planning combats planner overfitting in random MDPs.}
\label{f:epsilon}
\end{figure}

We see that manipulating either $\check{\gamma}$ or $\epsilon$ can be used to modulate the impact of planner overfitting. Which method to use in practice depends on the particular planner being used and how easily it is modified to use these methods.

\section{Decreased Policy Complexity}
\label{s:search}

In addition to indirectly controlling policy complexity via $\epsilon$ and $\gamma$, it is possible to control for the complexity via the representation of the policy itself. In this section, we look at varying the complexity of the policy in the context of model-based RL in which a model is learned and then a policy for that model is optimized via a policy search approach. Such an approach was used in the setting of helicopter control~\cite{ng03} in the sense that collected data in that work was used to build a model and a policy was constructed to optimize performance in this model (via policy search, in this case) and then deployed in the environment.

Our test domain was Lunar Lander, an environment with a continuous state space and discrete actions. The goal of the environment is to control a falling spacecraft so as to land gently in a target area. It consists of 8 state variables, namely the lander's $x$ and $y$ coordinates, $x$ and $y$ velocities, angle and angular velocities, and two Boolean flags corresponding to whether each leg has touched down. The agent can take 4 actions, corresponding to which of its three thrusters (or no thruster) is active during the current time step. The Lunar Lander environment is publicly available as part of the OpenAI Gym Toolkit~\cite{openaigym}.

We collected $40$k $200$-step episodes of data on Lunar Lander. During data collection, decisions were made by a policy-gradient algorithm. Specifically, we ran the REINFORCE algorithm with the state--value function as the baseline~\cite{Williams92,sutton00}. For the policy and value networks, we used a single hidden layer neural network with 16 hidden units and relu activation functions. We used the Adam algorithm~\cite{kingma14} with the default parameters and a step size of $0.005$. The learned model was a 3-layer neural net with ReLU activation functions mapping the agent's state (8 inputs corresponding to 8 state variables) as well as a one-hot representation of actions (4 inputs corresponding to 4 possible actions). The model consisted of two fully connected hidden layers with 32 units each and ReLU activations. We again used Adam and used step size $0.001$ to learn the model.

We then ran policy-gradient RL (REINFORCE) as a planner using the learned model. The policy was represented by a neural network with a single hidden layer. To control the complexity of the policy representation, we varied the number of units in the hidden layer from $1$ to $2000$.  Results were averaged over $40$ runs. Figure~\ref{f:policygradient} shows that increasing the size of the hidden layer in the policy resulted in better and better performance on the learned model (top line). However, after 250 or so units, the resulting policy performed less well on the actual environment (bottom line). Thus, we see that reducing policy complexity serves as yet another way to reduce planner overfitting.


\begin{figure}
\centering
\includegraphics[width=3.5in]{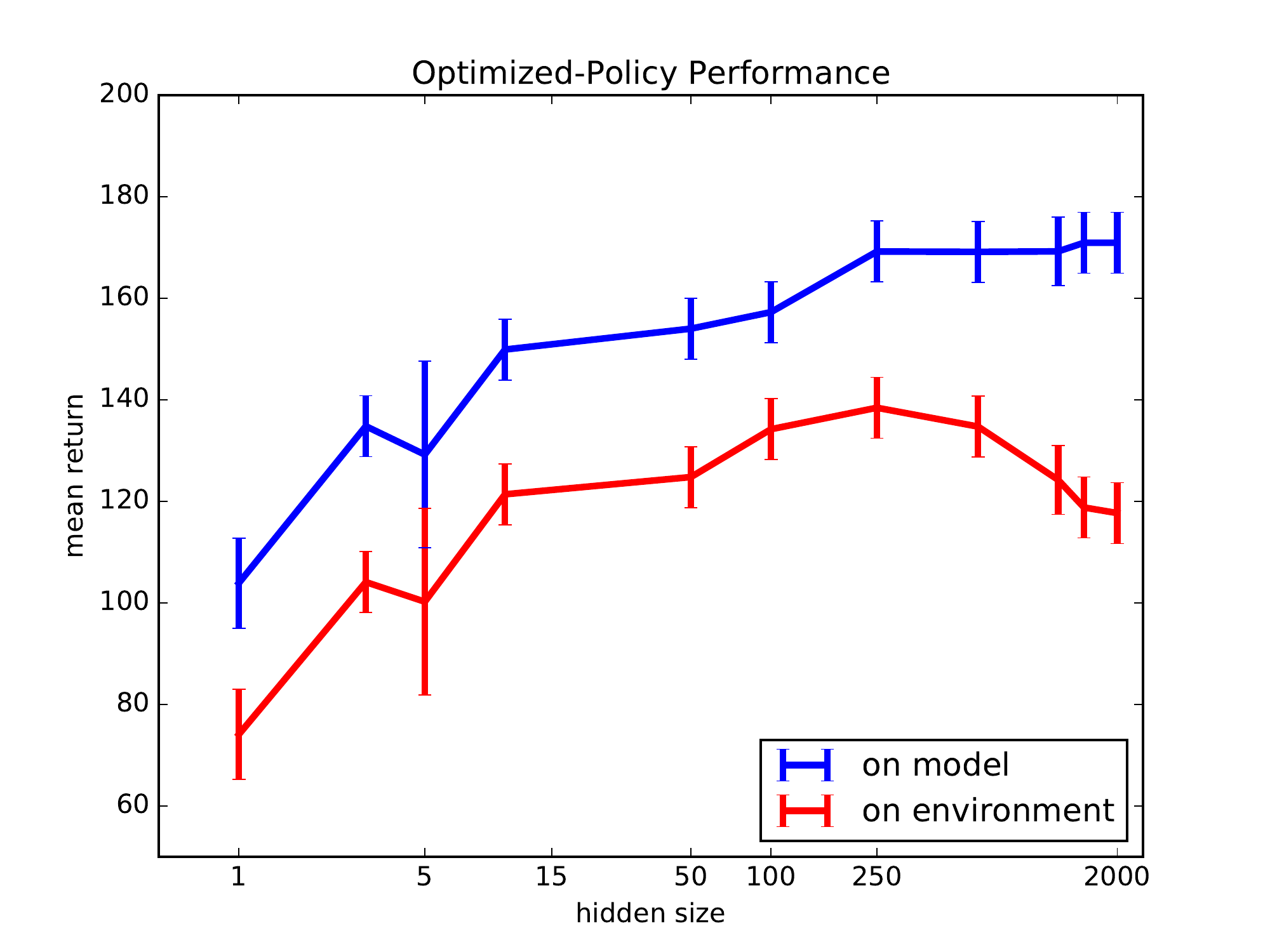}
\caption{Decreasing the number of hidden units used to represent a policy combats planner overfitting in the Lunar Lander domain.}
\label{f:policygradient}
\end{figure}

\section{Related Work}
\label{s:related}


Prior work has explored the use of regularization in reinforcement learning to mitigate overfitting. We survey some of the previous methods according to which function is regularized: (1) value, (2) model, or (3) policy.

\subsection{Regularizing Value Functions}

Many prior approaches have applied regularization to value function approximation, including Least Squares Temporal Difference learning~\cite{bradtke1996linear}, Policy Evaluation, and the batch approach of Fitted $Q$-Iteration (FQI)~\cite{ernst2005tree}.


\citet{Kolter2009a} applied regularization techniques to LSTD~\cite{bradtke1996linear} with an algorithm they called LARS-TD. In particular, they argued that, without regularization, LSTD's performance depends heavily on the number of basis functions chosen and the size of the data set collected. If the data set is too small, the technique is prone to overfitting. They showed that $L_1$ and $L_2$ regularization yield a procedure that inherits the benefits of selecting good features while making it possible to compute the fixed point. Later work by \citet{liu2012regularized} built on this work with the algorithm RO-TD, an $L_1$ regularized off policy Temporal Difference Learning method. \citet{Johns2010} cast the $L_1$ regularized fixed-point computation as a linear complementarity problem, which provides stronger solution-uniqueness guarantees than those provided for LARS-TD.  \citet{petrik2010feature} examined the approximate linear programming (ALP) framework for finding approximated value functions in large MDPs. They showed the benefits of adding an $L_1$ regularization constraint to the ALP that increases the error bound at training time and helps fight overfitting.

\citet{Farahmand2008} and \citet{Farahmand2009a} focused on regularization applied to Policy Iteration and Fitted $Q$-Iteration (FQI)~\cite{ernst2005tree} and developed two related methods for Regularized Policy Iteration, each leveraging $L_2$ regularization during the evaluation of policies for each iteration. The first method adds a regularization term to the Least Squares Temporal Difference (LSTD) error~\cite{bradtke1996linear}, while the second adds a similar term to the optimization of Bellman residual minimization~\cite{baird1995residual,schweitzer1985generalized,williams1993tight} with regularization~\cite{loth2007sparse}. Their main result shows finite convergence for the $Q$ function under the approximated policy and the true optimal policy. A method for FQI adds a regularization cost to the least squares regression of the $Q$ function. Follow up work~\cite{Farahmand2008a} expanded Regularized Fitted $Q$-Iteration to planning. That is, given a data set $\mc{D} = \langle (s_1, a_1, r_1, s_1'), \ldots, (s_m, a_m, r_m, s_m') \rangle$ and a function family $\mc{F}$ (like regression trees), FQI approximates a $Q$ function through repeated iterations of the following regression problem:
\begin{eqnarray*}
\lefteqn{\hat{Q}_{t+1}}\\
&=& \argmin_{Q \in F} \sum_{i=1}^m \left[r_i + \gamma \max_{a' \in A} \hat{Q}_t(s_i,a) - Q(s_i', a_i)\right]^2 \\ & & + \lambda \text{Pen}(\hat{Q}),
\end{eqnarray*}
where $\lambda \text{Pen}(\hat{Q})$ imposes a regularization penalty term and $\lambda$ is a regularization coefficient. They prove bounds relating this regularization cost to the approximation error in $\hat{Q}$ between iterations of FQI.


\citet{fahramand2011ms} and \citet{farahmand2011regularization} focused on a problem relevant to our approach---regularization for $Q$ value selection in RL and planning. They considered an offline setting in which an algorithm, given a data set of experiences and set of possible $Q$ functions, must choose a $Q$ function from the set that minimizes the true Bellman error.
They provided a general complexity regularization bound for model selection, which they applied to bound the approximation error for the $Q$ function chosen by their proposed algorithm, \textsc{BErMin}.

\subsection{Regularizing Models}

In model-based RL, regularizaion can be used to improve estimates of $R$ and $T$ when data is finite or limited.

\citet{Taylor2009a} investigated the relationship between Kernelized LSTD~\citet{xu2005kernel} and other related techniques, with a focus on regularization in model-based RL. Most relevant to our work is their decomposition of the Bellman error into transition and reward error, which they empirically show offers insight into the choice of regularization parameters.

\citet{Bartlett2009a} developed an algorithm, \textsc{Regal}, with optimal regret for weakly communicating MDPs. \textsc{Regal} heavily relies on regularization; based on all prior experience, the algorithm continually updates a set $\mc{M}$ that, with high probability, contains the true MDP. Letting $\lambda^*(M)$ denote the optimal per-step reward of the MDP $M$, the traditional optimistic exploration tactic would suggest that the agent should choose the $M'$ in $\mc{M}$ with maximal $\lambda^*(M')$. \textsc{Regal} {\it also} includes a regularization term to this maximization to prevent overfitting based on the experiences so far, resulting in state-of-the-art regret bounds. 

\subsection{Regularizing Policies}

The focus of applying regularization to policies is to limit the complexity of the policy class being searched in the planning process. It is this approach that we adopt in the present paper.

\citet{Somani2013} explored how regularization can help online planning for Partially Observable Markov Decision Processes (POMDPs). They introduced the \textsc{Despot} algorithm (Determinized Sparse Partially Observable Tree), which constructs a tree that models the execution of all policies on a number of sampled scenarios (rollouts). However, the authors note that \textsc{Despot} typically succumbs to overfitting, as a policy that performs well on the sampled scenarios is not likely to perform well in general. The work proposes a regularized extension of \textsc{Despot}, \textsc{R-Despot}, where regularization takes the form of balancing between the performance of the policy on the samples with the complexity of the policy class. Specifically, \textsc{R-Despot} imposes a regularization penalty on the utility of each node in the belief tree. The algorithm then computes the policy that maximizes regularized utility for the tree using a bottom up dynamic programming procedure on the tree. The approach is similar to ours in that it also limits policy complexity through regularization, but focuses on regularizing utility instead of regularizing the use of a transition model. Investigating the interplay between these two approaches poses an interesting direction for future work.  In a similar vein, \citet{Thomas2015} developed a batch RL algorithm with a probabilistic performance guarantee that limits the complexity of the policy class as a means of regularization.

\citet{Petrik2008a} conducted analysis similar to \citet{jiang15}. Specifically, they investigated the situations in which using a lower-than-actual discount factor can improve solution quality given an approximate model, noting that this procedure has the effect of regularizing rewards. The work also advanced the first bounds on the error of using a smaller discount factor.


\section{Conclusion}
\label{s:conclusions}

For three different regularization methods---decreased discounting, increased exploration, and decreased policy complexity, we found a consistent U-shaped tradeoff between the size of the policy class being searched and its performance on a learned model. Future work will evaluate other methods such as drop out and early stopping.

The plots that varied $\epsilon$ and $\gamma$ were quite similar, raising the possibility that perhaps epsilon-greedy action selection is functioning as another way to decrease the effective horizon depth using in planning---chaining together random actions makes future states less predictable and therefore carry less weight. Later work can examine whether jointly choosing $\epsilon$ and $\gamma$ is more
effective than setting only one at a time.

More work is needed to identify methods that can learn in much larger domains~\cite{bellemare13}. One concept worth considering is adapting regularization non-uniformly to the state space. That is, it should be possible to modulate the complexity of policies considered in parts of the state space where the model is more accurate, allowing more expressive plans is some places than others.



\newpage
\bibliography{reg}

\bibliographystyle{icml2018}

\end{document}